\documentclass{article}
\PassOptionsToPackage{numbers}{natbib}
\usepackage[preprint]{neurips_2019}

\usepackage[utf8]{inputenc} 
\usepackage[T1]{fontenc}    
\usepackage{hyperref}       
\usepackage{url}            
\usepackage{booktabs}       
\usepackage{amsfonts}       
\usepackage{nicefrac}       
\usepackage{microtype}      
\usepackage{graphicx}
\usepackage{subcaption}
\usepackage{booktabs} 
\usepackage{bm}
\usepackage{amsmath}
\usepackage{amssymb}
\usepackage{amsthm}
\usepackage{hyperref}

\newcommand{\X}{\mathbf{x}}
\newcommand{\Z}{\mathbf{z}}
\newcommand{\PHI}{\bm{\phi}}
\newcommand{\THETA}{\bm{\theta}}
\newcommand{\XI}{\bm{\xi}}
\newcommand{\norm}[1]{\left\lVert#1\right\rVert}
\newtheorem{theorem}{Theorem}
\newtheorem{lemma}[theorem]{Lemma}

\title{Interpreting Rate-Distortion of Variational Autoencoder and Using Model Uncertainty for Anomaly Detection}

\author{%
  Seonho Park $\:\:\:$ George Adosoglou $\:\:\:$ Panos M. Pardalos \\
  Department of Industrial and Systems Engineering\\ 
  University of Florida\\
  Gainesville, Florida, USA\\
  \texttt{\{seonhopark,g.adosoglou,pardalos\}@ufl.edu} \\
}

\begin{document}

\maketitle

\begin{abstract}
Building a scalable machine learning system for unsupervised anomaly detection via representation learning is highly desirable.
One of the prevalent methods is using a reconstruction error from variational autoencoder (VAE) via maximizing the evidence lower bound. 
We revisit VAE from the perspective of information theory to provide some theoretical foundations on using the reconstruction error, and finally arrive at a simpler and more effective model for anomaly detection. 
In addition, to enhance the effectiveness of detecting anomalies, we incorporate a practical model uncertainty measure into the metric.
We show empirically the competitive performance of our approach on benchmark datasets.

\end{abstract}

\section{Introduction}\label{sec:introduction}
Autoencoders have been widely used in many machine learning applications not only to reduce the noise from the input to learn representations but also to reconstruct the output with the salient information of the input.
These autoencoders learn common information of the inputs by mapping to the latent representations in an unsupervised manner.
When it comes to anomaly detection, using the reconstruction error of various autoencoders to discern anomalies has been widely and successfully employed \cite{hawkins2002outlier,sakurada2014anomaly,an2015variational,marchi2015novel,zhou2017anomaly}, even though using reconstruction error lacks its theoretical foundations. 
One of the autoencoders whose theoretical basis comes from variational inference is variational autoencoders (VAEs) \cite{kingma2013auto}.
VAEs try to minimize the difference between the true posterior and the variational posterior via maximizing the evidence lower bound (ELBO) with respect to the neural networks based encoder and decoder.
After training, we expect that the ELBO approximates the marginal likelihood of the data.

\begin{figure}[t]
\vskip 0.2in
\begin{center}
\centerline{\includegraphics[width=4.5in]{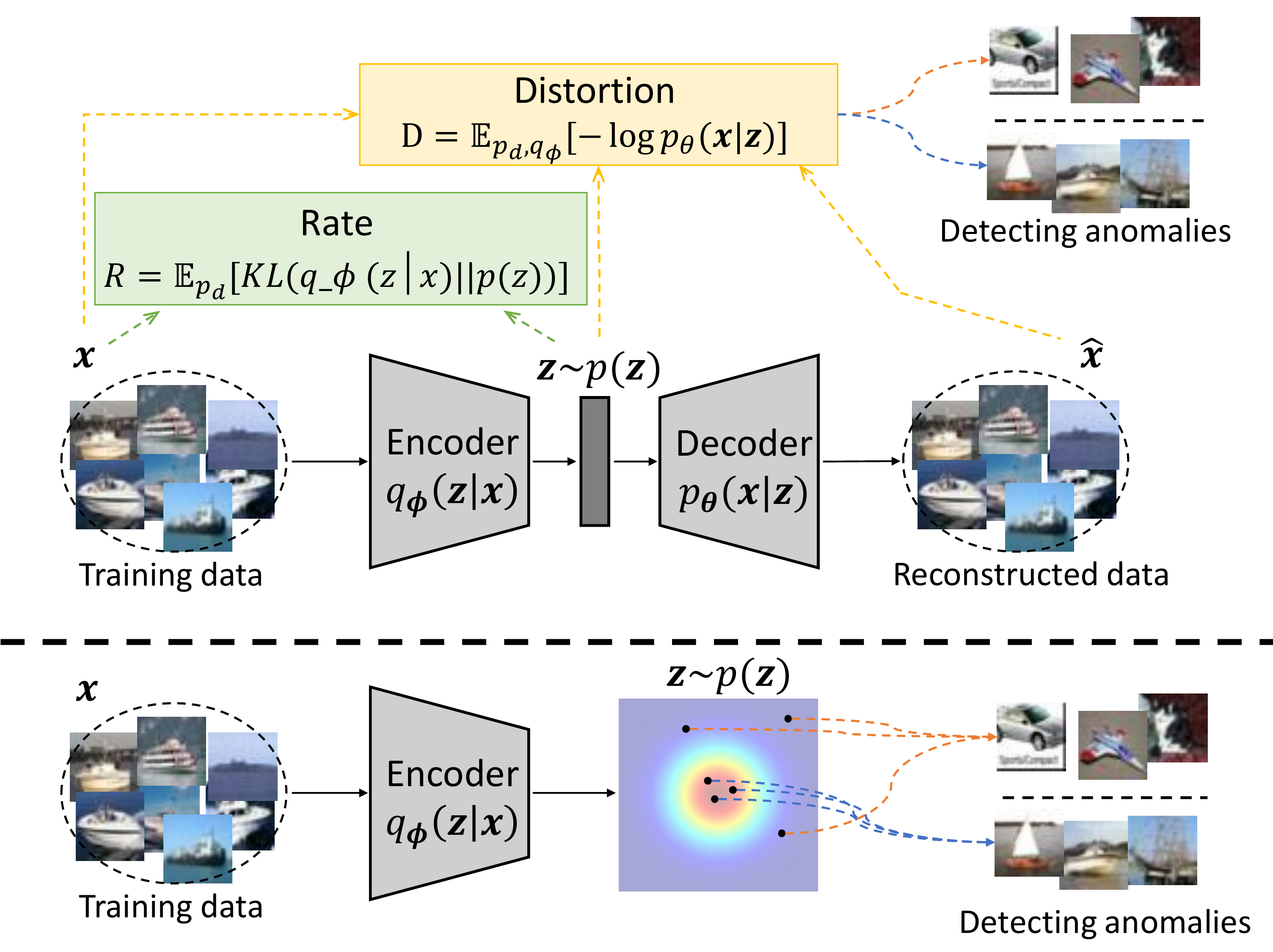}}
\caption{Overview of the proposed approach for anomaly detection. 
\textit{Top}: it shows the VAE based anomaly detection approach from the perspective of rate-distortion theory. The reconstruction error (distortion measure) is used as a metric of anomaly detection.
\textit{Bottom}: we only use the encoder (without the decoder) of VAEs to identify anomalies.
}
\label{fig:fig1}
\end{center}
\vskip -0.3in
\end{figure}

In this work, we revisit VAEs from the perspective of rate-distortion theory \cite{berger2003rate,alemi2017fixing} to elucidate the roles of the two terms: the rate and distortion.
Also, we argue that various autoencoders including $\beta$-VAE \cite{higgins2017beta} can be explained with the trade-off between the rate and distortion in this perspective.
Then, for the purpose of anomaly detection, we show that using only the encoder is more efficient to approximate the marginal likelihood, and finally, we arrive at a much simpler and more efficient model to discern anomalies.

Moreover, in order to enhance the performance of detecting anomalies, we incorporate the model uncertainty into our anomaly detection score. 
Since anomalies are unseen when the model is trained in an unsupervised setting, model uncertainty can capture the anomalies for which the model's confidence is low.
Finally, we justify our approaches with theoretical analyses as well as experiments on benchmark datasets.\footnote{One can reach out to the public implementation for whole experiments via \url{https://github.com/seonho-park/PGN_anomaly_detection}}

To summarize, we make the following contributions for anomaly detection.
\begin{itemize}
\setlength\itemsep{0.1em}
    \item With the theoretical foundation on variational inference and rate-distortion theory, we elucidate that the VAE-based anomaly detection system aim to approximate the marginal probability of the data.
    \item We propose that using the encoder only is more efficient and simpler than VAE's reconstruction error to measure anomaly score. 
    \item We incorporate model uncertainty into the metric to enhance anomaly detection performance. 
    \item We provide theoretical and empirical basis on our approach for anomaly detection.
\end{itemize}

\section{Problem Definition} \label{sec:problem_definition}
What we aim to do in this work is to derive an anomaly score $s(\mathbf{x})$ to indicate whether a given query datapoint $\mathbf{x}$ is anomalous or not.
More formally, with a scalar threshold $\gamma$, anomaly score $s(\mathbf{x})$ should distinguish anomalous instances as,
\begin{align*}\label{eq:anomaly_score}
    &s(\mathbf{x}) \geq \gamma \Rightarrow \text{anomalous}\\
    &s(\mathbf{x}) < \gamma \Rightarrow \text{normal}
\end{align*}
We cannot anticipate which anomalies come to the system in the future so it is reasonable to assume that an anomaly detection model is learned in an unsupervised way, that is, when training, we only have access to normal data and when testing we can access \textit{contaminated} data consisting of both normal and anomalous instances. 
This setting is also referred to as one class classification \cite{scholkopf2001estimating,ruff2018deep}.

\section{Information Theoretical Interpretation of VAE}\label{sec:VAE}
In this section, we revisit VAEs \cite{kingma2013auto} in the context of information theory to clarify the terms of the loss function of VAEs.

\paragraph{Variational Autoencoder}
Let us assume that we have a dataset $\mathbf{X}=\{\X^{(i)}\}_{i=1}^N$ consisting of normal datapoints $\X\in \mathcal{X}$, i.i.d. sampled.
The datapoints in $\mathbf{X}$ are realized by a random process, $p^*(\X|\Z)p^*(\Z)$, where $p^*(\Z)$ and $p^*(\X|\Z)$ are a true prior over latent variables and a true likelihood, respectively.
Also, we assume that the latent random variable $\Z\in\mathcal{Z}\subseteq\mathbb{R}^J$ follows a true prior $p^*(\Z)$.

Given an input $\X$, a variational posterior (which is also referred to as an encoder) is derived to approximate a true posterior via KL divergence and the corresponding marginal log-likelihood can be expressed with variational inference (VI) as
\begin{equation}\label{eq:vae_1}
    \log p(\X) = KL(q_{\PHI}(\Z|\X)||p(\Z|\X))+\mathcal{L}_{VI}(\PHI, \THETA;\X)
\end{equation}
where $q_{\PHI}(\Z|\X)$ is a neural network model parameterized by parameters $\PHI$.
The first RHS term of the above equation (Eq. \ref{eq:vae_1}) is the KL divergence between the variational and true posterior. 
The KL divergence is always nonnegative and it is zero if and only if the variational posterior is exactly equivalent to the true posterior which is intractable to compute directly.
Because the KL divergence is nonnegative, we could say the second RHS term is the lower bound of the marginal log-likelihood, $\log p(\X)$ which is fixed.
This second RHS term can be elaborated as
\begin{equation}\label{eq:elbo_1}
\begin{aligned}
    \mathcal{L}_{VI}(\PHI, \THETA;\X) = -KL(q_{\PHI}(\Z|\X)||p(\Z))+\\
    \mathbb{E}_{q_{\PHI}(\Z|\X)}\left[\log p_{\THETA}(\X|\Z)\right]
\end{aligned}
\end{equation}
where $p_{\THETA}(\X|\Z)$ is a variational approximation (decoder) to a true likelihood, parameterized by parameters $\THETA$ and $p(z)$ is an approximation to $p^*(z)$.
The first RHS term acts as a regularizer of $q_{\PHI}$ and the second RHS term corresponds to the negative reconstruction error.
By taking an expectation w.r.t. the empirical data distribution $p_d(\X)$, VAEs seek to maximize the evidence lower bound (ELBO) to minimize the KL divergence between the variational posterior and true posterior as
\begin{equation}\label{eq:elbo_2}
\begin{aligned}
    \max_{\PHI, \THETA} \mathbb{E}_{p_d(\X)}\left[ -KL(q_{\PHI}(\Z|\X)||p(\Z))\right]\\
    + \mathbb{E}_{p_d(\X)}\left[\mathbb{E}_{q_{\PHI}(\Z|\X)}\left[\log p_{\THETA}(\X|\Z)\right]\right]
\end{aligned}
\end{equation}

The ELBO consists of two terms.
The first term in Eq.\ref{eq:elbo_2} can be interpreted as compression loss of the input information.
If the first term is high (as $q_{\PHI}(\Z|\X)$ approaches $p(\Z)$, it means that the latent code compresses the input so well that the salient information of the input disappears.\footnote{Sometimes, compression is also referred to as disentanglement because what we aim to get as a latent representation is usually a disentangled representation and manipulate some elements of the latent vector to tweak the reconstruction readily. Please see \cite{burgess2018understanding,kim2018disentangling,higgins2017beta} for more details.}
The second term is the expected negative reconstruction error, which represents the (negative) difference between the input and the output from the decoder.
Thus, the ELBO can be interpreted as the trade-off between the compression loss (how much information can be lost in the latent space) and the (negative) reconstruction error (how much information can be retrieved from the decoder). 

\paragraph{VAE as Lossy Compression}
From the perspective of the rate-distortion theory \cite{berger2003rate}, we revisit VAE to elucidate the roles of the terms of the ELBO.
We derive two terms the \textit{rate} and \textit{distortion}, which correspond to negative compression loss and the reconstruction error, respectively.

Based on the previous work \cite{alemi2017fixing}, we can rewrite the VAE problem as:
\begin{equation}\label{eq:rate_distortion_loss}
    \min_{\PHI, \THETA} R(\X,\Z | \PHI) + D\left(\X,\Z | \PHI, \THETA \right)
\end{equation}
where
\begin{align}
    R\left(\X,\Z | \PHI \right)&=\mathbb{E}_{p_d(\X)}\left[ KL(q_{\PHI}(\Z|\X)||p(\Z)) \right]\label{eq:def_rate}\\
    D\left(\X,\Z | \PHI, \THETA \right)&=\mathbb{E}_{p_d(\X)}\left[ \mathbb{E}_{q_{\PHI}(\Z|\X)}\left[-\log p_{\THETA}(\X|\Z) \right]  \right]\label{eq:def_distortion}
\end{align} 
The rate, $R$, is the expected value of the rate measure, $KL(q_{\PHI}(\Z|\X)||p(\Z))$.
The rate represents the expectation of the KL divergence between the encoder and prior.

$D$ is the distortion, the expected value of distortion measure, $d(\X,\Z)=\mathbb{E}_{q_{\PHI}}[-\log p_{\THETA}]$ representing the reconstruction error.
Note that the rate only depends on the parameters $\PHI$ of the encoder, while the distortion depends on both $\PHI$ and $\THETA$.

Also, we would like to introduce the data entropy, $H$, as
\begin{equation}\label{eq:data_entropy}
    H(\X) = \mathbb{E}_{p_d(\X)}\left[ -\log p(\X) \right]
\end{equation}
Given $R$, $D$, and $H$, the expectation w.r.t. $p_d(\X)$ of Eq.\ref{eq:vae_1} can be rewritten as,
\begin{equation}\label{eq:entropy_rate_distortion}
    H = \mathbb{E}_{p_d(\X)}\left[-KL(q_{\PHI}(\Z|\X)||p(\Z|\X))\right] + R + D
\end{equation}
From the nonnegative property of the KL divergence, we can say $H\leq R+D$ where the equality holds if and only if the variational posterior equals to the true posterior, i.e., $q_{\PHI}(\Z|\X^{(i)})=p(\Z|\X^{(i)}),\:\forall \X^{(i)}\in\mathbf{X}$.
Then, $H=R+D$.
This represents a theoretical lower bound of $R+D$ and is depicted as the red solid line in Fig. \ref{fig:fig2}.
In VAE, we may not achieve this ideal case, $H=R+D$, because of the limited finite families of parameters, the approximated prior and noises in the given dataset.
Instead, we seek to find the (information) distortion-rate function (curve) by solving the following optimization problem:
\begin{equation}\label{eq:distortion_rate_function}
\begin{aligned}
    \min_{\PHI,\THETA} D\\
    \text{subject to } R\leq \bar{R}
\end{aligned}
\end{equation}
where $\bar{R}$ denotes an upper limit of the rate.
In order to optimize both without taking $\bar{R}$, we can take the Lagrangian of Eq. \ref{eq:distortion_rate_function} with a Lagrange multiplier $\beta>0$ as\footnote{One could think this formulation is to minimize two objectives $R$ and $D$. Then, we can achieve the Pareto frontier of the objectives, which is corresponding to the distortion-rate function in rate-distortion theory as well.},
\begin{equation}\label{eq:beta_vae}
    \min_{\PHI, \THETA} D+\beta R
\end{equation}
which resembles the $\beta$-VAE objective \cite{higgins2017beta}.

\begin{figure}[t]
\vskip 0.2in
\begin{center}
\centerline{\includegraphics[width=3.0in]{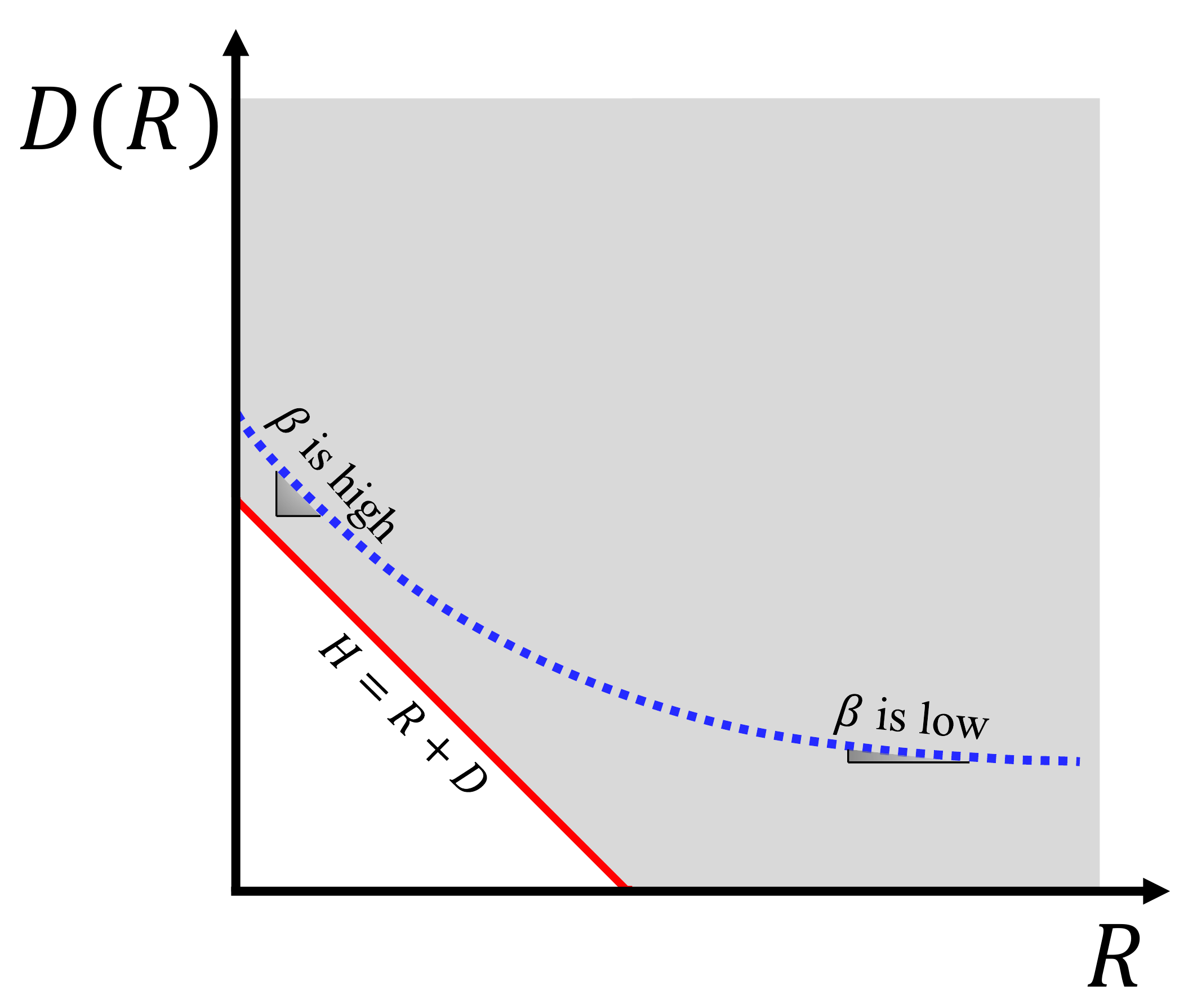}}
\caption{Schematic view of distortion-rate function. A red line corresponds to the theoretical lower bound of the rate and distortion. By varying $\beta$ of $\beta$-VAE, we could achieve the points on a blue dashed curve, the sub-optimal distortion-rate function, which is best achievable with VAEs.}
\label{fig:fig2}
\end{center}
\vskip -0.2in
\end{figure}
\setlength{\belowcaptionskip}{-20pt}

Varying $\beta$ in $\beta$-VAE, we can get the distortion-rate function depicted as a blue dashed curve in Fig. \ref{fig:fig2}.
Even though the curve is not explicitly formed except for some known simple examples, it is known that the distortion-rate function is convex and monotonically non-increasing.
With these properties, the Lagrangian multiplier can be interpreted as a negative slope of the distortion-rate function.
This $\beta$ plays a role to balance the rate and the distortion.
When $\beta$ is high, we can get the point where the rate is low and the distortion is high.
Whereas, when $\beta$ is low, we can get the point where the rate is high and the distortion is low.
Because the joint distribution $p(\X,\hat{\X})$ of lossy compression is composed of the encoder and decoder, it is noted that distortion cannot be zero even when the rate is high enough \cite{tishby2015deep}.

\section{Method}\label{sec:method}
\paragraph{Zero Rate Setting to Approximate Data Distribution}
In this section, we first introduce a method to approximate the marginal log-likelihood of the data and derive the anomaly score considering model uncertainty as well as the approximate data distribution to detect anomalies.

The main purpose of using VAEs is to encode the salient information of the data to the latent space and reconstruct the output $\hat{\X}$. 
When it comes to anomaly detection, even though using the reconstruction error of autoencoders shows great performances empirically, it is not so straightforward and lacks some theoretical foundations.
From Eq. \ref{eq:distortion_rate_function}, if we set $\bar{R}$ to be zero then the learned rate has to be zero.
$R=0$ means from the definition of the rate that $q_{\PHI}(\Z|\X)=p(\Z)$ for all datapoints in training dataset.
The latent variable $\Z$ does not store any particular information of the individual datapoint and the decoder seeks to give outputs via the stochastic decoder and the reparameterization trick \cite{kingma2013auto} which resemble the empirical data distribution, $p_d(\X)$.
This means that the decoded outputs of different inputs are widely distributed as the training datapoints are.
Precisely, the zero rate corresponding distortion can be rewritten with the definition of the distortion (Eq. \ref{eq:def_distortion}) as
\begin{equation}\label{eq:distortion_zero_rate}
    D = -\int d\X p_d(\X) \int d\Z p(\Z) \log p_{\THETA}(\X|\Z)
\end{equation}
Thus, the gap between $H$ and $D$ when $R=0$ can be decreased when we use sufficiently large $p_d(\X)$ and a powerful decoder $p_{\THETA}$.
Equivalently, the difference of intercepts of the blue dashed curve and the red line in Fig. \ref{fig:fig2} can be interpreted by means of insufficient data representations of dataset, poor decoder performance, and implicitly wrong selection of prior and its dimension.

From this observation, we could conclude that, for the purpose of anomaly detection, we can achieve the distortion-rate function by setting  $\beta=\infty$ as well.
Then, we only need the encoder of the autoencoder to estimate the marginal log-likelihood of the data.
As a consequence, we solve the following problem to estimate the marginal log-likelihood of the data:
\begin{equation}\label{eq:gpn}
    \min_{\PHI} \mathcal{L}(\PHI) = \mathbb{E}_{p_d(\X)}\left[ KL(q_{\PHI}(\Z|\X)||p(\Z)) \right]
\end{equation}
With this formulation, since the decoder $p_{\THETA}$ vanishes, we could have a more concise model (only with the encoder) rather than using the reconstruction error which requires both the encoder and the decoder.
The model can be learned with first or second order stochastic optimization methods such as stochastic gradient descent, ADAM \cite{kingma2014adam}, AdaGrad \cite{duchi2011adaptive}, Newton or its variant methods \cite{bottou2018optimization,park2020combining} via direct backpropagation.

\paragraph{Prior Generating Networks}
We can break down the rate as reported in \cite{hoffman2016elbo} as
\begin{equation}
    \mathbb{E}_{p_d}\left[KL(q_{\PHI}(\Z|\X)||p(\Z))\right] = I(\X;\Z) + KL(q_{\PHI}(\Z)||p(\Z))
\end{equation}
where $I$ denotes the mutual information between $\X$ and $\Z$ under the joint distribution $q_{\PHI}(\X,\Z) = q_{\PHI}(\Z|\X)p_d(\X)$.
Also, $q_{\PHI}(\Z)$ is known as the aggregated variational posterior \cite{tomczak2017vae,kim2018disentangling} and can be attained as $q_{\PHI}(\Z)=\mathbb{E}_{p_d}\left[q_{\PHI}(\Z|\X)\right]$.
Also, when we set $p(\Z)=q_{\PHI}(\Z)$, the rate is identical to the mutual information $I(\X;\Z)$. 
This setting of the prior is referred to as the VampPrior \cite{tomczak2017vae}.
The mutual information $I(\X;\Z)$ is upper bounded by the rate $R$ as above.
Therefore, the zero rate setting means that the latent codes do not have any information about normal datapoints in the training dataset.
Intuitively, the rate measure $KL(q_{\PHI}(\Z|\X)||p(\Z))$ of anomalous instance is higher than that of normal instances. 

In our experiments, the prior $p(\Z)$ is defined by an isotropic multivariate Gaussian, $\mathcal{N}(\mathbf{0},\mathbf{1})$ as in \cite{kingma2013auto}.
Also let the encoder be the neural network based model of which the outputs are the mean and standard deviation of the isotropic multivariate Gaussian, i.e., $q_{\PHI}(\Z|\X^{(i)})=(\bm{\mu}^{(i)}, \bm{\sigma}^{(i)})$.
Thus $\mathcal{L}(\PHI)$ in Eq.\ref{eq:gpn} can be simplified as
\begin{equation}
\begin{aligned}
    \mathbb{E}_{p_d(\X)}\left[ KL(q_{\PHI}(\Z|\X)||p(\Z)) \right]\\
    = \frac{1}{N}\sum_{i=1}^{N}\sum_{j=1}^{J}\left(\log\frac{1}{\sigma_j^{(i)}} + \frac{(\sigma_j^{(i)})^2+\mu_j^{(i)})^2}{2}-\frac{1}{2} \right)
\end{aligned}
\end{equation}
In our experiments, we have used this setting and we call this \textit{prior generating networks (PGN)} in what follows. This is named after the fact that the neural network based encoder merely aims to approximate the prior $p(z)$.

\paragraph{Anomaly Score with Model Uncertainty}
Disregarding the expectation with respect to $p_d$, from Eq. \ref{eq:entropy_rate_distortion} we can estimate the log probability of a query input $\X^*$. 
Note that as the encoder converges to the prior we cannot reconstruct the input, the distortion measure should be constant. 
Hopefully, if we assume that the KL divergence between the variational posterior and true posterior is sufficiently low, so negligible, then the log probability of $\X^*$ is proportional to the negative KL divergence between the variational posterior and prior, i.e., $\log p(\X^*) \propto - KL(q_{\PHI}(\Z|\X^*)||p(\Z))$.

Our model $q_{\PHI}(\Z|\X)$ can be deterministic and the data distribution, $p_d(\X)$, can only impose some stochasticity into the model.
The model can gain more stochasticity by employing a random noise $\XI$ into the model as $q_{\PHI}(\Z|\X,\XI)$.
One practical way to do this is to use MC dropout \cite{gal2015dropout}.
Inserting dropout layers \cite{srivastava2014dropout} in the model, MC dropout estimates the first and second moments by Monte Carlo samplings with $T$ stochastic forward passes.

MC dropout is one of the most prevalent methods used for capturing model uncertainty.
Model uncertainty, also referred to as \textit{epistemic} uncertainty, comes from the lack of knowledge of the data.
It includes uncertainties generated by the situation where the model does not have enough knowledge and/or experience on the data\footnote{The protocol of OOD detection is similar to that of anomaly detection, where we have access to normal data when training and distinguish the anomalies (or data came from other datasets). However, OOD detection is usually conducted as a byproduct to assess the ``confidence`` of the system for classifications or regressions while the system for anomaly detection is merely for it. Please see \cite{hendrycks2016baseline} for more details on OOD detection.}.
From the problem definition, we also assume that the model is trained only on the normal datapoints so that the model uncertainty can capture anomalies by generating higher uncertainties on them.

Let us define $\XI^{(t)}$ as the \textit{t}-th realization of the random noise $\XI$.
Also, its elements are i.i.d and sampled from the Bernoulli distribution with the dropout probability $p$.
With the $T$ stochastic forward passes, $KL(q_{\PHI^*}(\Z|\X)||p(\Z))=KL(\mathbb{E}_{\XI}\left[q_{\PHI^*}(\Z|\X,\XI)\right]||p(\Z))$ and $\mathbb{E}_{\XI}\left[q_{\PHI^*}(\Z|\X,\XI)\right] \simeq \frac{1}{T}\sum_{t=1}^{T}q_{\PHI^*}(\Z|\X,\XI^{(t)})$.
For simplicity, we only impose $\XI$ into the mean output $\bm{\mu}$ of $q_{\PHI}$.
As a result, we propose and use the following anomaly score metric.
Given a query point $\X^*$ and learned parameters $\PHI^*$,
\begin{align}
    s(\X^*) = \frac{1}{T}\sum_{t=1}^T KL\left(q_{\PHI^*}(\Z|\X^*, \XI^{(t)})||p(\Z)\right)=\label{eq:anomaly_score}\\
    KL\left(\frac{1}{T}\sum_{t=1}^{T}q_{\PHI^*}(\Z|\X,\XI^{(t)})||p(\Z)\right) + 
    \text{Variation}[\bm{\mu}(\X^*)]\label{eq:anomaly_score2}
\end{align}
where $\bm{\mu}(\X^*)$ is an abbreviation of the mean output of $q_{\PHI}(\Z|\X^*)$ and Variation$[\cdot]$ denotes the model uncertainty measure.
It is noted that Eq. \ref{eq:anomaly_score} is equivalent to the summation of MC dropout based estimation and the measured model uncertainty. 
So Eq. \ref{eq:anomaly_score} as an anomaly score is efficient to capture both the mean KL value with additionally imposed stochasticity and model uncertainty via MC dropout.

\paragraph{Theoretical Analysis of Model Uncertainty}
We elucidate the reason that Eq. \ref{eq:anomaly_score} is equal to Eq. \ref{eq:anomaly_score2}.
To this end, we first revisit the following lemma regarding the gap of Jensen's inequality.

\begin{lemma}[The Gap of Jensen's inequality]\label{lemma:gap_jensen}
Let $x$ be a one dimensional random variable and $p(x\in(a,b))=1$ where $-\infty \leq a<b\leq \infty$.
Let $\varphi(x)$ be a twice differentible function on $(a,b)$. 
Then,
\begin{equation}\label{eq:gap_jensen}
\begin{aligned}
    \frac{\inf_x\varphi''(x)\mathrm{Var}(x)}{2}&\leq \mathbb{E}[\varphi(x)]-\varphi(\mathbb{E}[x]) \\
    &\leq \frac{\sup_x\varphi''(x)\mathrm{Var}(x)}{2}
\end{aligned}
\end{equation}
\end{lemma}
\begin{proof}
Please refer to the proof of the theorem 1 in \cite{becker2012variance} and \cite{liao2019sharpening}.
\end{proof}

This lemma implies that when the function $\varphi$ is strictly convex the Jensen's inequality gap represents the variance of the random variable.
With this, we can derive the following theorem to justify Eq. \ref{eq:anomaly_score}, the proposed anomaly score.

\begin{theorem}[PGN anomaly score measure]\label{thm:anomaly_score}
From a finite parameters set $\PHI$, let us assume that we have learned parameters $\PHI^*$.
Let $\bm{\mu}$ and $\bm{\sigma}$ be outputs of the encoder $q_{\PHI}$ and $p(\Z)$ be an isotropic multivariate Gaussian prior.
Also, assume that $\bm{\mu}$ involves a random noise $\XI$.
Then, given an arbitrary input data $\X^*$, the following equality holds
\begin{equation}\label{eq:theorem}
\begin{aligned}
    \mathbb{E}_{\XI}\left[KL(q_{\PHI^*}(\Z|\X^*,\XI)||p(\Z))\right] = \\
    KL(\mathbb{E}_{\XI}\left[q_{\PHI^*}(\Z|\X^*,\XI)\right]||p(\Z))+ \alpha\sum_{j=1}^J\mathrm{Var}_{\XI}(\mu_j(\X^*|\XI))
\end{aligned}
\end{equation}
\end{theorem}
\begin{proof}
Let us define $\varphi(\mu_j):=KL(\mathcal{N}(\mu_j,\sigma_j)||p(z_j))$ where $p(z_j) = \mathcal{N}(0,1)$. From the fact that $q_{\PHI}$ is finite valued and twice differentiable function based on the neural networks, KL divergence with a fixed $p(z_j)$ is also twice differentiable and strongly convex with respect to $\mu_j$, i.e., $\varphi''(\mu_j)>0,\:\forall \mu_j$.
Let us denote that $\inf\varphi''(\mu_j)=m_j$ and $\sup\varphi''(\mu_j)=M_j$ where $0< m_j < M_j<\infty\:\forall j\in\{1,J\}$.
Because of Lemma \ref{lemma:gap_jensen}, the following equality holds:
\begin{equation}\label{eq:gap_jensen2}
    \frac{m_j}{2}\mathrm{Var}(\mu_j)\leq \mathbb{E}[\varphi(\mu_j)]-\varphi\left(\mathbb{E}[\mu_j]\right)\leq \frac{M_j}{2}\mathrm{Var}(\mu_j)
\end{equation}
Summing Eq. \ref{eq:gap_jensen2} upto $J$ and taking $\alpha$ such that $\min_j\{m_j\}<\frac{2\alpha}{J}<\max_j\{M_j\}$ finalize the proof.
\end{proof}
From Eq. \ref{eq:anomaly_score2}, model uncertainty metric, Variation$[\cdot]$, is proportional to $\sum_{j=1}^J\mathrm{Var}_{\XI}(\mu_j(\X^*|\XI))$ of Eq. \ref{eq:theorem}.
The expected value can be approximated by $T$ stochastic forward passes.
Because it only needs $T$ inferences and does not involve any further computations, incorporating model uncertainty into our anomaly score metric (Eq. \ref{eq:anomaly_score}) is so practical.

\paragraph{Relationship to Deep SVDD}
We would like to highlight that our PGN learning (Please see Eq. \ref{eq:gpn}) generalizes the popular anomaly detection method, Deep SVDD \cite{ruff2018deep}.
Deep SVDD is to minimize the hypersphere in the latent space.
If an instance lies out of the hypersphere, it is deemed anomalous.
Given an input dataset $\mathbf{X}$ and a prescribed center $\mathbf{c}\in\mathcal{Z}$, \textit{One-Class Deep SVDD} \cite{ruff2018deep} trains the neural network based model $f_{\mathbf{w}}$, parameterized by parameters $\mathbf{w}$, as

\begin{equation}\label{eq:one_class_deep_svdd}
    \min_{\mathbf{w}} \frac{1}{N}\sum_{i=1}^N\norm{f_{\mathbf{w}}(\Z|\X^{(i)})-\mathbf{c}}^2+\lambda \Omega(\mathbf{w})
\end{equation}
where the second term represents a weight decay regularizer with a hyperparameter $\lambda>0$.
After training, the anomaly score can be calculated as
\begin{equation}
    s(\X) = \norm{f_{\mathbf{w}^*}(\Z|\X) - \mathbf{c}}^2
\end{equation}
where $\mathbf{w}^*$ are the learned parameters.

From Eq.\ref{eq:gpn}, suppose that we set the prior $p(\Z)$ to an isotropic multivariate Gaussian, $\mathcal{N}(\mathbf{c}, \epsilon \mathbf{I})$, with $\epsilon\ll1 $ and that $q_{\PHI}$ only gives the estimated mean of the Gaussian with a fixed variance, i.e., $q_{\PHI}(\Z|\X^{(i)}) = \mathcal{N}(\Z|\bm{\mu}^{(i)}, \epsilon\mathbf{I})$.
Then, Eq,\ref{eq:gpn} can be rewritten as
\begin{equation}\label{eq:gpn_dsvdd}
\begin{aligned}
    \min_{\PHI} &\frac{1}{N}\sum_{i=1}^{N}\sum_{j=1}^{J}\left(\log\frac{\epsilon}{\epsilon} + \frac{\epsilon+(\mu_j^{(i)}-c_j)^2}{2\epsilon}-\frac{1}{2} \right)\\
    & = \frac{1}{N}\sum_{i=1}^{N}\sum_{j=1}^{J} \left(\frac{(\mu_j^{(i)}-c_j)^2}{2\epsilon}\right)
\end{aligned}
\end{equation}
Therefore, disregarding the weight decay term in Eq.\ref{eq:one_class_deep_svdd}, we claim that PGN is, in some sense, a general formulation of the Deep SVDD and it provides different perspective to Deep SVDD, which does not rely on the previous kernel based methods such as OC-SVM \cite{scholkopf2001estimating} or SVDD \cite{tax2004support}.

\section{Related Works}\label{sec:related_works}
\paragraph{Deep Anomaly Detection based on Neural Networks}
Outlier detection using replicator neural networks \cite{hawkins2002outlier} is, to the best of our knowledge, the first anomaly detection that uses neural networks where the reconstruction error is used as an anomaly score named `\textit{outlyingness score}'.
They introduced the replicator neural networks, feed-forwarding multi-layer perceptron neural networks with three hidden layers that forms the compressed latent representations and tries to reconstruct the inputs.
Many recent approaches that are based on the reconstruction error of autoencoders \cite{sakurada2014anomaly,an2015variational,dau2014anomaly} are also based on the same philosophical reasons.
To enhance the performance of detecting anomalies in a huge amount of complex and high-dimensional data, the different types of deep autoencoders such as VAEs \cite{kingma2013auto}, adversarial autoencoders (AAEs) \cite{makhzani2015adversarial}, denoising autoencoder \cite{vincent2010stacked}, and deep convolutional autoencoders \cite{masci2011stacked} have been equipped.

In a similar vein, generative adversarial networks (GANs) \cite{goodfellow2014generative} have been also used as architectures of anomaly detection while using the reconstruction error as an anomaly score. The examples of this approach include AnoGAN \cite{schlegl2017unsupervised} and OCGAN \cite{perera2019ocgan}.
Because GANs focus on the powerful data generation, these anomaly detection approaches suggest some methodologies to increase the reconstruction error for anomalies.
There are also deep neural networks based anomaly detection methods by \textit{One-class classification} such as Deep SVDD \cite{ruff2018deep} and OC-NN \cite{chalapathy2018anomaly}, which are inspired by kernel based methods, SVDD \cite{tax2004support} and OC-SVM \cite{scholkopf2001estimating}, respectively.  

Generative Probabilistic Novelty Detection (GPND) \cite{pidhorskyi2018generative} seeks to approximate the probability density of the data to distinguish anomalies.
To this end, they train the AAE-like architecture to learn the manifold structure of data distribution. 

\paragraph{Understanding VAE with Information Theory}
\textit{Information bottleneck} \cite{tishby2000information} can be utilized to understand representation learning such as variational autoencoders (VAEs).
VAEs is regarded as lossy compression in the context of rate-distortion theory.
This way of understanding gives a different theoretical point of view to understand VAEs. \cite{alemi2017fixing,brekelmans2019exact,blau2019rethinking}.
In \cite{lastras2019information}, the author argues that considering rate-distortion theory is the key to understanding representation learning and studies the possibility to optimize the marginal prior which is usually treated as fixed.
Alemi et al. \citet{alemi2017fixing} try to understand ELBO with the rate-distortion theory based framework which is similar to our work.
They show that decoupling ELBO to the rate and distortion helps understand the behavior of VAEs.

\section{Experiments}\label{sec:experiments}
\subsection{Distortion-Rate Functions of VAE}\label{subsec:dr_functions}

\begin{figure*}[ht]
\vskip 0.2in
\begin{subfigure}[c]{0.333\textwidth}
\centering
\includegraphics[height=1.3in]{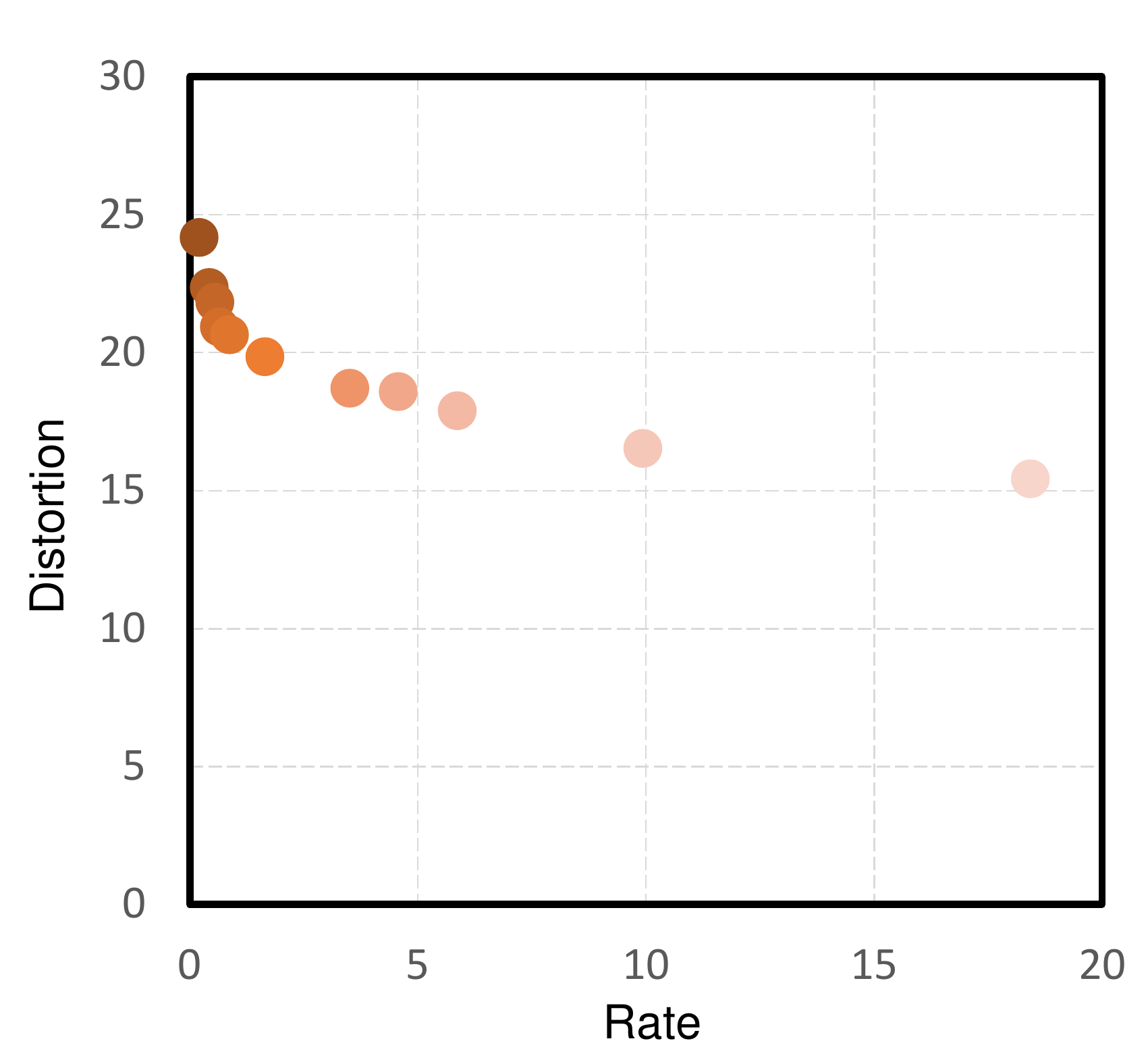}
\caption{MNIST}\label{subfig:mnist}
\end{subfigure}\hspace{-8mm}
\begin{subfigure}[c]{0.333\textwidth}
\centering
\includegraphics[height=1.3in]{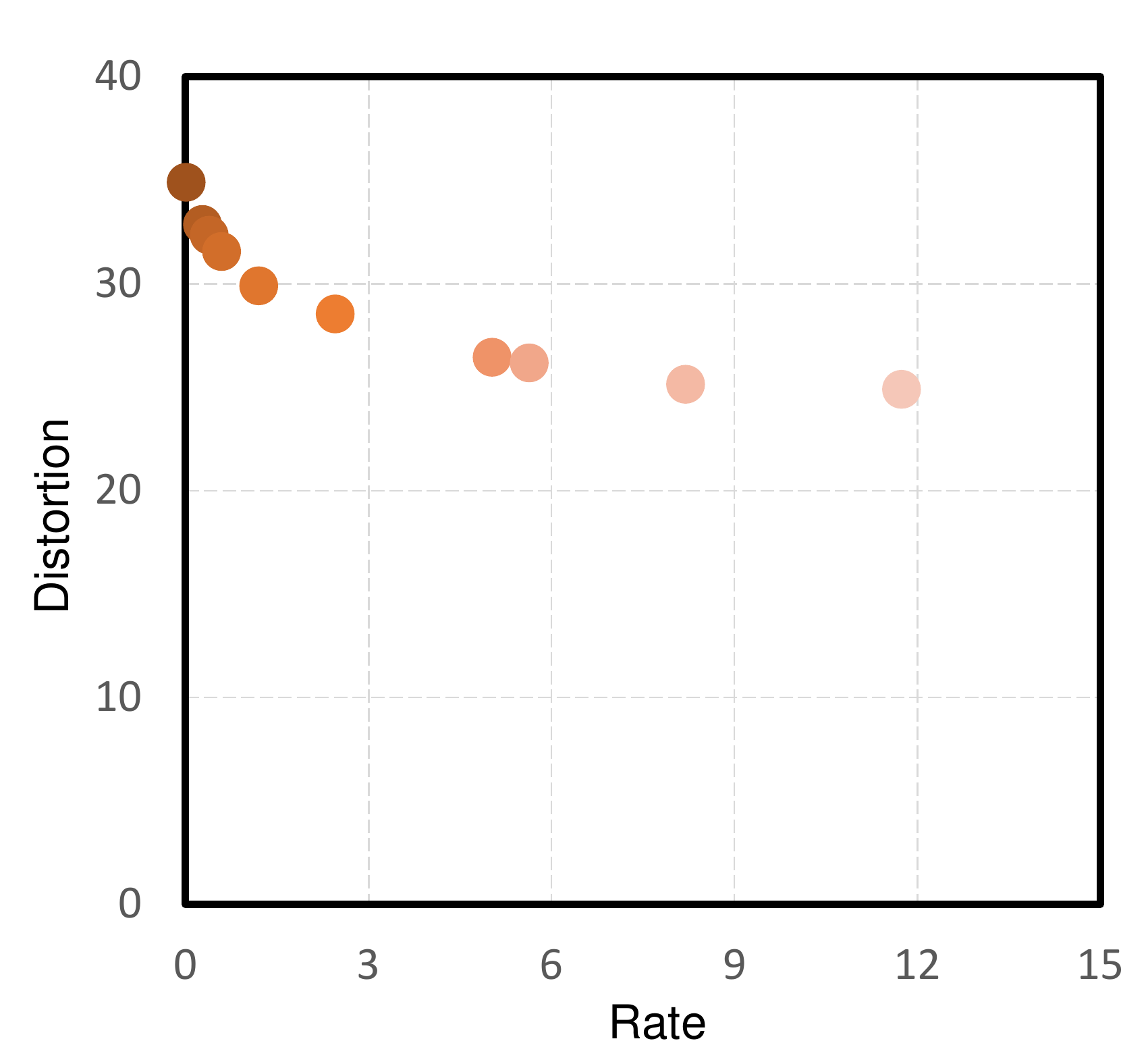}
\caption{FMNIST}\label{subfig:fmnist}
\end{subfigure}\hspace{-4mm}
\begin{subfigure}[c]{0.333\textwidth}
\centering
\includegraphics[height=1.3in]{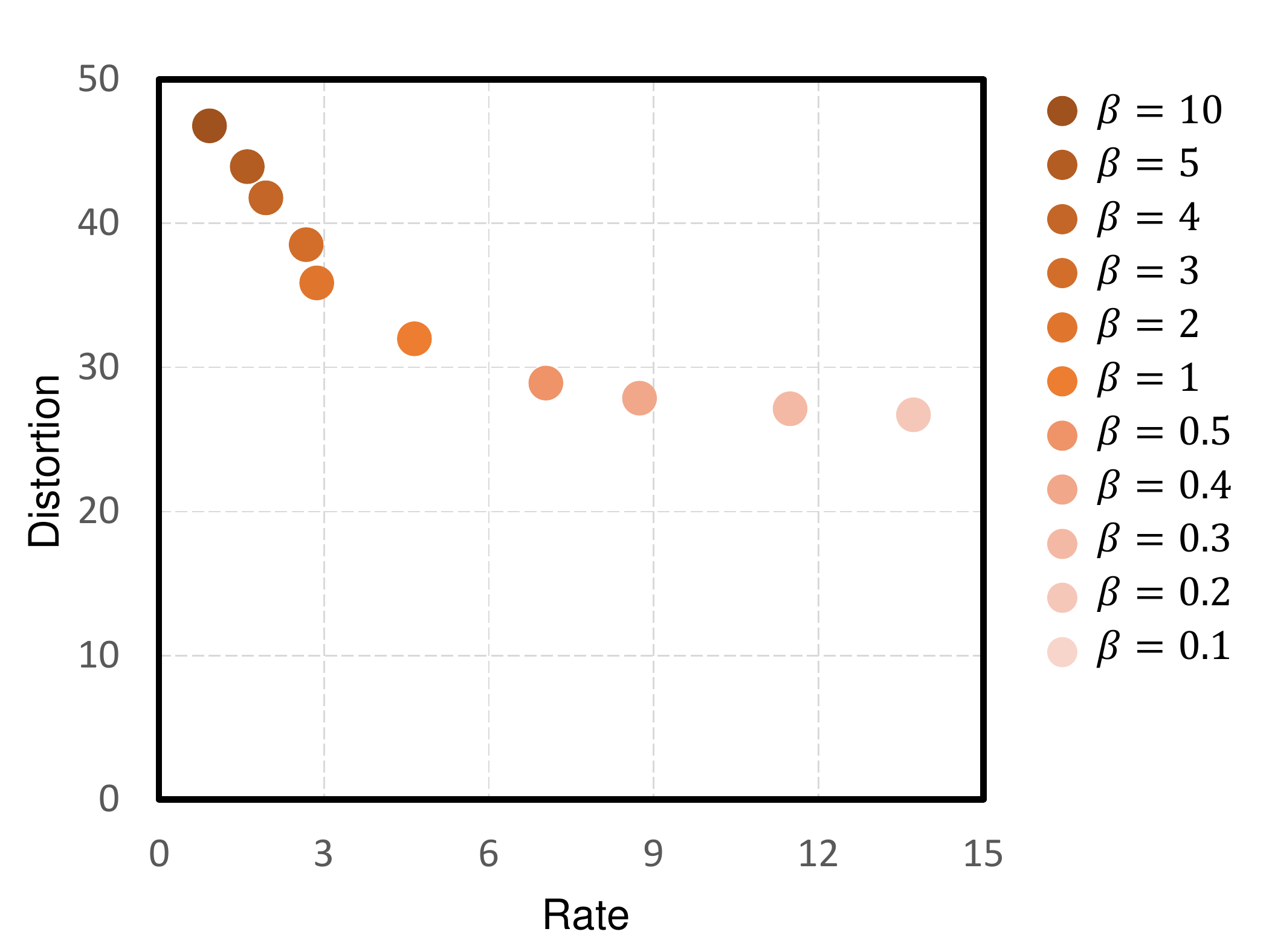}
\caption{CIFAR10}\label{subfig:cifar10}
\end{subfigure}
\vskip 0.2in
\caption{Distortion-rate functions on MNIST, FMNIST and CIFAR10 datasets. These show that VAE architectures which we use for anomaly detection experiments have sub-optimal distortion-rate functions empirically.
}\label{fig:fig3}
\vskip 0.2in
\end{figure*}

\paragraph{Experiment Settings} 
To verify our assumption that VAEs have the sub-optimal distortion-rate functions as depicted in Fig. \ref{fig:fig2}, we investigate the rate and distortion by varying $\beta$ of $\beta$-VAE (Eq. \ref{eq:distortion_rate_function}) on MNIST \cite{lecun1998mnist}, Fashion MNIST (FMNIST) \cite{xiao2017fashion} and CIFAR10 \cite{krizhevsky2014cifar} datasets.
We have used the official pre-split training dataset to train the model without any label information for all datasets. 
This work was similarly conducted by previous works \cite{alemi2017fixing,brekelmans2019exact}.
We use the LeNet-like convolutional and deconvolutional autoencoders.
On MNIST, the encoder contains two convolutional modules; $8\times(5,1,2)$ convolutional layers and $4\times(5,1,2)$ convolutional layers where a format is of (kernel size, stride, padding).
The convolutional modules are followed by batch normalization, Leaky ReLU with $\alpha=0.1$ and $2\times2$ max pooling. 
On CIFAR10, it contains three convolutional modules; $32\times(5,1,2)$ convolutional layers followed by $64\times(5,1,2)$ and $128\times(5,1,2)$ convolutional layers.
The dimension of the latent space is set to 32 for MNIST dataset, and 128 for CIFAR10 dataset.
The decoders for both datasets contain symmetrical transposed convolutional modules to each encoder.
We have used the reparameterization trick \cite{kingma2013auto} with $L=10$.
Adam optimizer \cite{kingma2014adam} is used for 1000 epochs with a learning rate of $5\mathrm{e}{-5}$, a weight decay of $1\mathrm{e}{-4}$ and batch size of 200.
We did not conduct any data augmentations or pre-processing without normalizing the images to $[0,1]$.
The runs were performed with different $\beta$ values from $\beta\in\{10,5,4,3,2,1,0.5,0.4,0.3,0.2,0.1\}$.

\paragraph{Results} Fig. \ref{fig:fig3} shows the distortion-rate functions on datasets.
Depending on $\beta$, we arrive at different $R$ and $D$ points.
As expected, when $\beta$ is increased, the resulting $R$ values get decreased, which is what we have expected in Section \ref{sec:VAE}.
It is noted that as $\beta$ decreased, $R+D$ gets increased meaning distortion-rate functions on these datasets are sub-optimal, which is due to the joint distribution $p(\X,\hat{\X})$.
By means of using the powerful encoder-decoder or imposing an appropriate prior, we could shrink the gap to the optimal rate-distortion tradeoff more, but this is out of our scope. 
One can also find similar results from \cite{alemi2017fixing}.

\subsection{Anomaly Detection Performances}
\paragraph{Baselines}
We compare our method, PGN, with other deep anomaly detection baselines for the anomaly detection task.
We have considered three baselines using the reconstruction error as their anomaly score including naive autoencoder (AE), variational autoencoder (VAE) \cite{kingma2013auto} and adversarial autoencoder (AAE) \cite{makhzani2015adversarial}. 
We also have considered Deep SVDD (DSVDD) \cite{ruff2018deep} and GPND \cite{pidhorskyi2018generative} as baselines. 

\paragraph{Datasets and Experiment Protocol}
We have used MNIST, FMNIST, and CIFAR10 as benchmark datasets.
We used official split training and test sets for all datasets.
We took data of one class in the pre-split training set as our training set (thus, our training set consists of all normal instances) while our test set is the same as the pre-split test set.
The test instances except for instances having the class label of the training set are deemed anomalies.
MNIST, FMNIST, and CIFAR10 have 10 classes each. 
So we conducted 10 independent anomaly detection experiments for each dataset.
The number of anomalies is about 9 times more than that of normal instances in test datasets.

\paragraph{Architectures and Experiment Settings}
The architectures of AE, AAE, VAE are the same as those in Section \ref{subsec:dr_functions}.
For VAE, we also have used the reparameterization trick \cite{kingma2013auto} with $L=10$.
AAEs train the model in adversarial learning and need the discriminator to discriminate between the generated latent variables of inputs and values sampled from the marginal prior.
This discriminator consists of three fully-connected layers with the number of weights of $J\times512$, $512\times256$, and $256\times1$, respectively, and each layer is followed by Leaky ReLU activations with $\alpha=0.1$.
Architectures of DSVDD and PGN are identical, the encoder of the autoencoder architecture. 
We also have considered \textit{hypersphere collapse} as reported in \cite{ruff2018deep} meaning that the model with bias terms can produce a trivial solution for DSVDD.
We found that hypersphere collapse can also occur in our PGN setting, so we did not use any bias terms in both PGN and DSVDD. 
Also, PGN outputs the mean and variance values of isotropic multivariate Gaussian.

For PGN, we have used $T=20$ for $T$ stochastic forward passes with the dropout probabiilty $p=0.5$.
It is noted that only mean outputs involve the MC dropout to consider the model uncertainty as described in Section \ref{sec:method} and Theorem \ref{thm:anomaly_score}.

It is reported that a pretraining with the autoencoder for DSVDD is helpful to enhance anomaly detection performance.
But we did not conduct any pretraining for all baselines and PGN for fair comparisons.
For AE, VAE, AAE, DSVDD and PGN, the Adam optimizer \cite{kingma2014adam} is used with a weight decay of $1\mathrm{e}{-4}$.
The learning rate is initialized to $1\mathrm{e}{-4}$ and reduced by a factor of 10 at 75th epochs for all datasets.
We train for 100 epochs and compare under the AUROC values on test datasets.
As stated in Section \ref{sec:problem_definition}, we treat the anomalies are positive and normal instances are negative, so that the anomaly scores of the methods are used for calculating AUROC values without any fixed thresholds.
For GPND, we have followed same experiment settings they recommended and provided\footnote{\url{https://github.com/podgorskiy/GPND}}.
For a data preprocessing, we only normalize the data to $[0,1]$ for both test and training datasets.

\begin{table*}[t]
\caption{Mean and std. dev. AUROCs [\%] with 10 different seeds on MNIST (\textit{Top}) and FMNIST (\textit{Bottom}) datasets.} \label{table:mnist_fmnist}
\vskip 0.15in
\small
\begin{center}
\begin{tabular}{@{}lcccccc@{}}
\toprule
Normal class & GPND  & DSVDD        & AE           & VAE          & AAE          & PGN (ours)          \\ \midrule
0    & 75.3$\pm$8.3 & 97.4$\pm$0.9 & \textbf{98.5$\pm$0.5} & 96.7$\pm$0.9 & 98.0$\pm$0.4 & 97.8$\pm$1.0 \\
1    & 96.2$\pm$2.5 & 99.6$\pm$0.2 & \textbf{99.9$\pm$0.0} & 99.8$\pm$0.0 & 99.8$\pm$0.0 & 99.6$\pm$0.1 \\
2    & 65.4$\pm$9.2 & 88.7$\pm$2.2 & 82.8$\pm$1.7 & 80.2$\pm$2.8 & 80.8$\pm$1.4 & \textbf{91.3$\pm$1.8} \\
3    & 68.9$\pm$6.9 & 89.4$\pm$1.3 & 90.3$\pm$1.5 & 88.9$\pm$0.6 & 89.7$\pm$0.9 & \textbf{91.1$\pm$1.3} \\
4    & 78.4$\pm$3.3 & 93.7$\pm$1.0 & 88.8$\pm$1.6 & 89.8$\pm$2.1 & 87.9$\pm$1.7 & \textbf{94.7$\pm$0.8} \\
5    & 69.3$\pm$5.6 & 87.1$\pm$2.6 & \textbf{92.3$\pm$1.3} & 89.9$\pm$2.3 & 90.3$\pm$2.0 & 89.8$\pm$2.0 \\
6    & 78.4$\pm$7.0 & 98.0$\pm$0.5 & 94.7$\pm$1.5 & 92.3$\pm$1.1 & 92.6$\pm$2.5 & \textbf{98.5$\pm$0.4} \\
7    & 83.8$\pm$4.8 & 94.1$\pm$1.0 & 94.6$\pm$0.9 & 92.8$\pm$0.9 & 93.9$\pm$0.6 & \textbf{94.9$\pm$0.8} \\
8    & 57.4$\pm$5.7 & 90.8$\pm$1.2 & 78.5$\pm$1.8 & 79.5$\pm$1.8 & 76.1$\pm$2.9 & \textbf{92.1$\pm$0.9} \\
9    & 77.1$\pm$3.9 & 95.9$\pm$0.6 & 91.8$\pm$1.5 & 90.6$\pm$2.2 & 90.1$\pm$2.2 & \textbf{96.7$\pm$0.3} \\
Avg.  & 75.0 & 93.5 & 91.2 & 90.1 & 89.9 & \textbf{94.7} \\ \midrule

T-shirt    & 77.2$\pm$7.9 & \textbf{90.4$\pm$1.1} & 88.1$\pm$0.5 & 87.5$\pm$0.6 & 87.9$\pm$0.6 & \textbf{90.4$\pm$3.4} \\
Trouser    & 95.8$\pm$1.4 & 98.5$\pm$0.2 & 97.8$\pm$0.2 & 96.9$\pm$0.3 & 98.1$\pm$0.1 & \textbf{98.6$\pm$0.1} \\
Pullover   & 78.1$\pm$5.6 & 85.8$\pm$3.1 & 83.7$\pm$0.6 & 83.8$\pm$0.8 & 80.7$\pm$1.6 & \textbf{86.1$\pm$5.2} \\
Dress      & 85.5$\pm$4.4 & 92.4$\pm$1.5 & 90.8$\pm$0.5 & 89.3$\pm$0.7 & 90.5$\pm$0.5 & \textbf{93.1$\pm$1.3} \\
Coat       & 77.8$\pm$4.3 & \textbf{89.2$\pm$1.4} & 86.7$\pm$0.5 & 85.1$\pm$0.8 & 86.6$\pm$0.6 & 87.7$\pm$6.5 \\
Sandal     & 89.4$\pm$0.9 & \textbf{89.4$\pm$0.6} & 83.3$\pm$1.2 & 82.3$\pm$1.4 & 83.5$\pm$1.2 & 89.2$\pm$0.5 \\
Shirt      & 76.3$\pm$3.5 & \textbf{80.6$\pm$1.7} & 78.7$\pm$0.3 & 79.0$\pm$0.6 & 77.5$\pm$0.4 & 80.3$\pm$2.3 \\
Sneaker    & 95.3$\pm$1.5 & 98.6$\pm$0.1 & 97.6$\pm$0.1 & 96.9$\pm$0.1 & 97.6$\pm$0.1 & \textbf{98.7$\pm$0.1} \\
Bag        & 68.1$\pm$4.2 & 91.1$\pm$1.9 & 75.2$\pm$1.4 & 75.3$\pm$2.1 & 75.5$\pm$2.3 & \textbf{92.2$\pm$5.5} \\
Ankle boot & 88.1$\pm$5.0 & \textbf{98.4$\pm$0.3} & 94.8$\pm$0.7 & 92.9$\pm$0.7 & 95.4$\pm$0.9 & 98.2$\pm$1.3 \\
Avg.       & 83.2 & 91.4 & 87.7 & 86.9 & 87.3 & \textbf{91.5} \\ 
\bottomrule
\end{tabular}
\end{center}
\vskip -0.1in
\end{table*}

\begin{table*}[ht!]
\vskip 0.5in
\caption{Mean and std. dev. AUROCs [\%] with 10 different seeds on CIFAR10 dataset.} \label{table:cifar10_result}
\begin{center}
\begin{tabular}{@{}lcccccc@{}}
\toprule
Normal class & GPND  & DSVDD        & AE           & VAE          & AAE          & PGN (ours)          \\ \midrule
Airplane   & 56.6$\pm$4.3 & 59.6$\pm$4.9  & 68.7$\pm$0.2 & 69.4$\pm$0.5 & 67.9$\pm$0.4 & \textbf{73.7$\pm$3.5} \\
Automobile & 55.2$\pm$3.8 & \textbf{56.4$\pm$2.4}  & 40.0$\pm$0.3 & 44.6$\pm$0.4 & 41.6$\pm$0.3 & 55.1$\pm$2.8 \\
Bird       & 55.1$\pm$2.4 & 64.9$\pm$1.8  & 65.0$\pm$0.2 & \textbf{65.1$\pm$0.2} & \textbf{65.1$\pm$0.2} & 64.3$\pm$3.0 \\
Cat        & 56.5$\pm$3.3 & 53.4$\pm$0.7  & 55.8$\pm$0.3 & 53.8$\pm$0.2 & 54.2$\pm$0.2 & \textbf{56.7$\pm$2.5} \\
Deer       & 69.2$\pm$1.8 & \textbf{72.4$\pm$3.1}  & 67.2$\pm$0.1 & 68.5$\pm$0.2 & 68.1$\pm$0.3 & 70.5$\pm$2.1 \\
Dog        & 53.9$\pm$2.0 & 53.5$\pm$2.4  & 56.2$\pm$0.2 & 53.9$\pm$0.2 & 54.4$\pm$0.2 & \textbf{60.3$\pm$3.8} \\
Frog       & 70.0$\pm$5.4 & \textbf{74.2$\pm$2.7}  & 55.4$\pm$0.3 & 61.1$\pm$0.5 & 58.7$\pm$0.4 & 72.1$\pm$3.5 \\
Horse      & \textbf{58.4$\pm$2.1} & 54.3$\pm$2.4  & 44.9$\pm$0.4 & 46.2$\pm$0.3 & 45.1$\pm$0.2 & 54.9$\pm$2.9 \\
Ship       & 64.8$\pm$4.0 & 67.6$\pm$2.2  & 74.5$\pm$0.3 & 74.5$\pm$0.4 & 73.4$\pm$0.3 & \textbf{75.5$\pm$2.8} \\
Truck      & 54.5$\pm$4.7 & \textbf{60.9$\pm$2.8}  & 41.8$\pm$0.2 & 45.2$\pm$0.5 & 42.7$\pm$0.3 & 60.8$\pm$2.2 \\
Avg. & 59.4 & 61.7 & 57.0 & 58.2 & 57.1 & \textbf{64.4}\\
\bottomrule
\end{tabular}
\end{center}
\vskip -0.1in
\end{table*}

\paragraph{Comparison Results}
Table \ref{table:mnist_fmnist} shows the results on MNIST and FMNIST datasets.
As shown, most of the cases, PGN gives the best results. 
AE gives competitive results among the reconstruction error based methods, AE, VAE, and AAE. 
AAE needs adversarial learning for training the generator (encoder) but it does not guarantee that it gives any meaningful differences with AE and/or VAE.
VAEs incorporate the rate into the loss function but anomaly detection conducts only with the distortion (reconstruction error) so VAE also does not give any advanced results.
Even though VAE gives the theoretical foundation on the autoencoder and can tweak the disentangled latent values to manipulate output easily, it does not have any advantages in anomaly detection tasks empirically. 
Table \ref{table:cifar10_result} gives the results on CIFAR10 dataset.
It shows that PGN gives the most competitive results on this dataset as well.
PGN and DSVDD empirically show very similar AUROC values along the classes because they share the same theoretical formulation, but we can say PGN is slightly but meaningfully better than DSVDD.
The main reason for this is that PGN detects anomalies with their involved model uncertainty metric, which can be captured by its anomaly score (Eq. \ref{eq:anomaly_score}).

\paragraph{Ablation Study }

\begin{figure*}[ht!]
\vskip 0.2in
\begin{subfigure}[c]{0.333\textwidth}
\centering
\includegraphics[width=0.9\columnwidth]{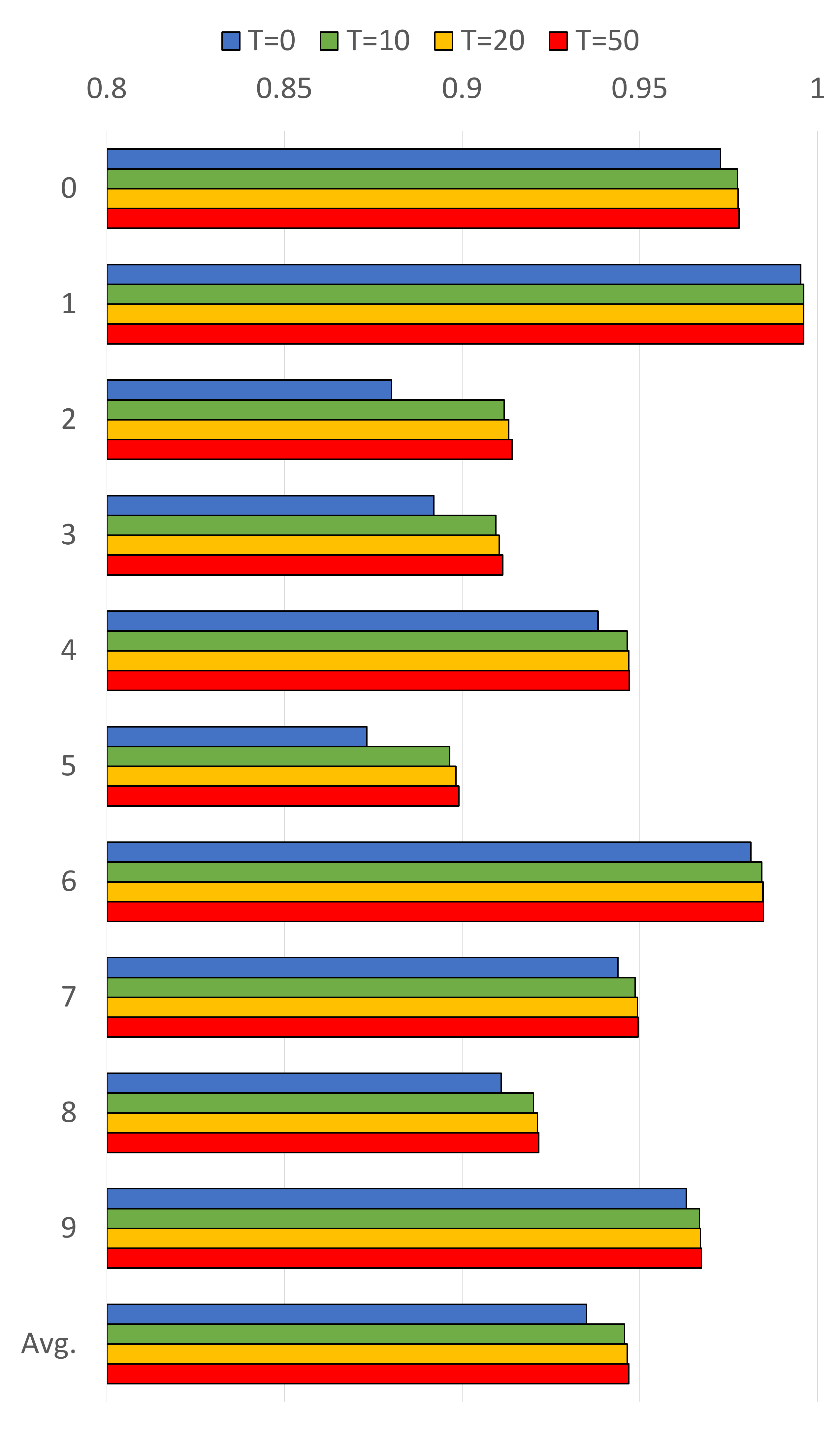}
\caption{MNIST}\label{subfig:ablation_mnist}
\end{subfigure}\hspace{-8mm}
\begin{subfigure}[c]{0.333\textwidth}
\centering
\includegraphics[width=0.9\columnwidth]{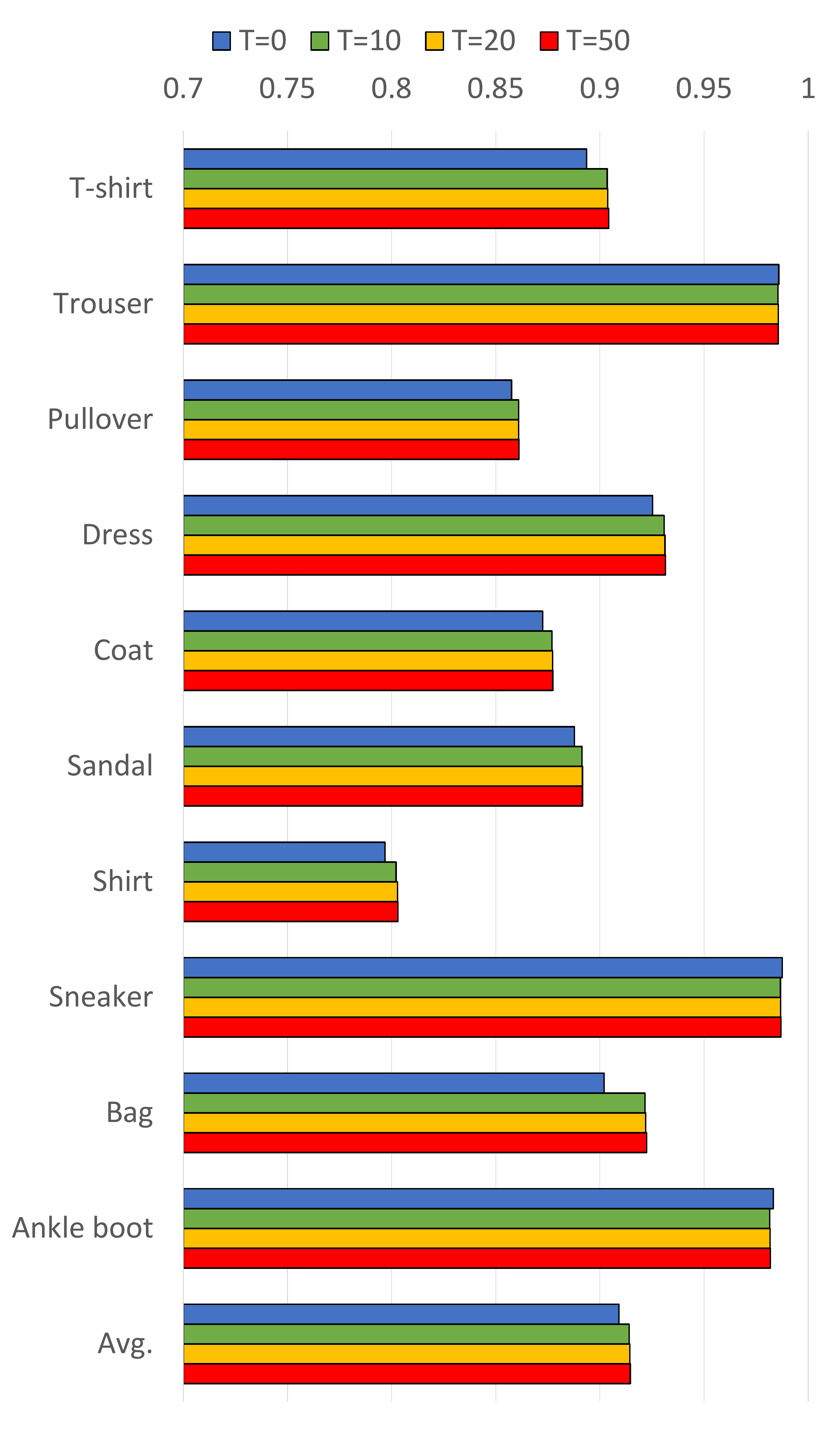}
\caption{FMNIST}\label{subfig:ablation_fmnist}
\end{subfigure}\hspace{-4mm}
\begin{subfigure}[c]{0.333\textwidth}
\centering
\includegraphics[width=0.9\columnwidth]{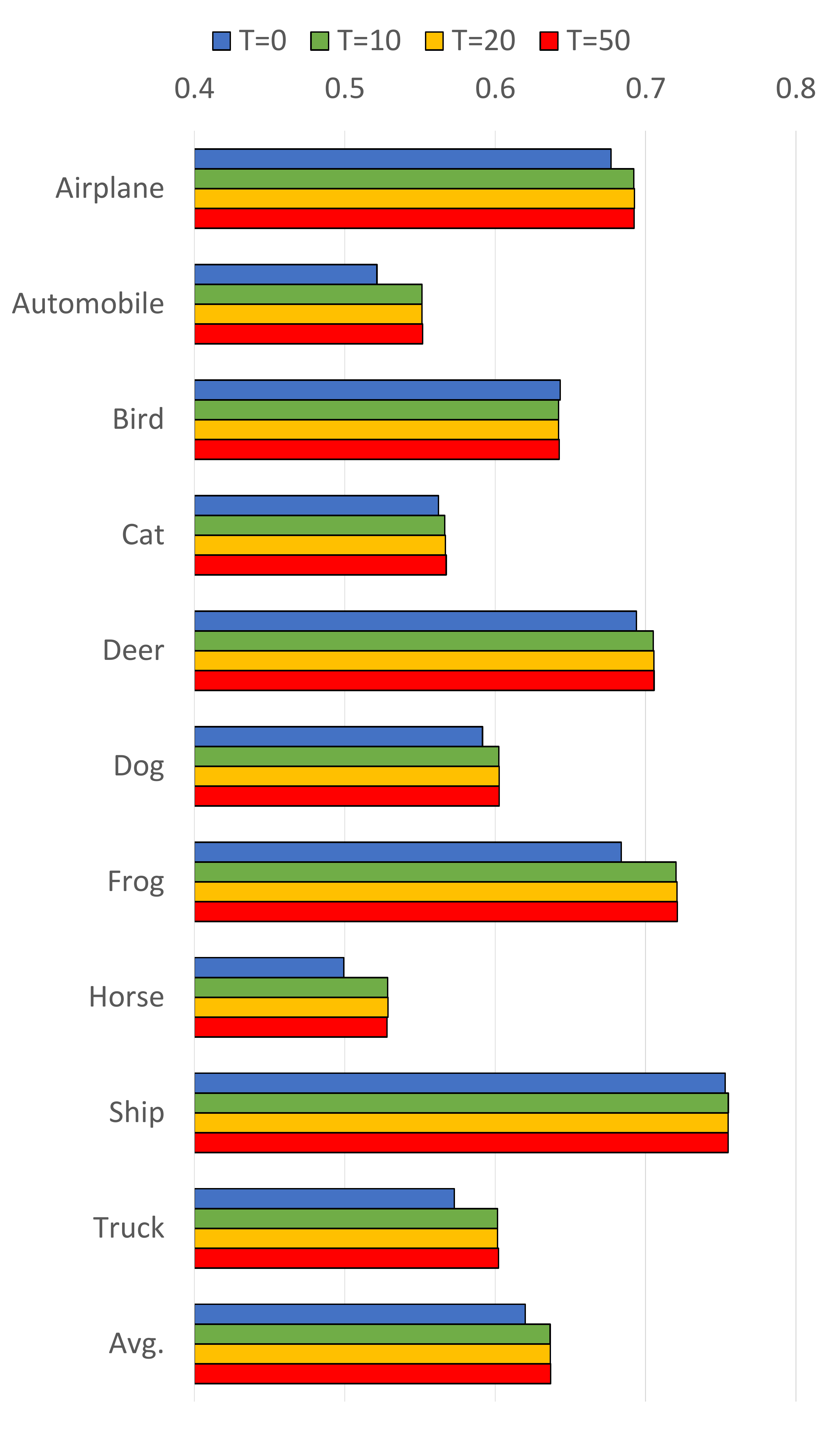}
\caption{CIFAR10}\label{subfig:ablation_cifar10}
\end{subfigure}
\vskip 0.2in
\caption{Ablation experiment results of model uncertainty metric of PGN on MNIST, FMNIST and CIFAR10 datasets. The values represent AUROCs [\%] with 10 different seeds.
}\label{fig:ablation}
\vskip 0.2in
\end{figure*}

To scrutinize the effectiveness of involving the model uncertainty we proposed, we performed an ablation study on three datasets.
The results are shown in Fig. \ref{fig:ablation}.
In Fig. \ref{fig:ablation}, $T=0$ corresponds to the standard dropout meaning that we turn off stochastic forward passes when testing thus it does not consider the model uncertainty involved in the input instance.
It shows that when $T\geq10$ the mean and standard deviation values are very similar so we can say that using $T=10$ is enough to leverage model uncertainty via MC dropout for these datasets.
As a result, when we use MC dropout, the AUROC values on three datasets get increased with significant margins.

\section{Conclusion}\label{sec:conclusion}
We would like to highlight that this work could connect the deep anomaly detection method with the theoretical foundations on variational inference and information theory.
We propose PGN that can capture anomalies by means of estimating data distribution and shows better result by incorporating MC dropout based model uncertainty metric.
We expect that this model could be better when equipped with more powerful architectures.

\bibliography{ref}
\bibliographystyle{unsrtnat}

\end{document}